\documentclass[11pt]{article}

\usepackage{libertine}
\usepackage[T1]{fontenc}
\usepackage{amsmath,amsthm}
\usepackage{booktabs}
\usepackage{multirow}
\usepackage{graphicx}
\usepackage{xcolor}

\usepackage{arxivs}
\usepackage{amsfonts,amssymb}
\usepackage{mathrsfs}

\usepackage{cite}
\usepackage{hyperref}
\usepackage{textcomp}
\usepackage{tabularx}
\usepackage{caption}
\usepackage{subcaption}
\usepackage{diagbox}
\usepackage{algorithm}
\usepackage{algorithmicx}
\usepackage{algpseudocode}
\usepackage{bm}

\usepackage[section]{placeins}   
\newtheorem{theorem}{Theorem}
\newtheorem{proposition}[theorem]{Proposition}

\theoremstyle{definition}
\newtheorem{definition}[theorem]{Definition}

\theoremstyle{remark}

\newcommand{\Eb}{\mathbb{E}}

\newcommand{\Rb}{\mathbb{R}}

\newcommand{\Ac}{\mathcal{A}}

\newcommand{\Fc}{\mathcal{F}}
\newcommand{\Jc}{\mathcal{J}}
\newcommand{\Kc}{\mathcal{K}}
\newcommand{\Lc}{\mathcal{L}}
\newcommand{\Pc}{\mathcal{P}}
\newcommand{\Sc}{\mathcal{S}}

\newcommand{\sys}{\text{sys}}

\newcommand{\kc}{\kappa}

\newcommand{\one}{\mathbb{I}}
\newcommand{\onen}{\mathbf{1}}

\title{Fairness in Multi-class Multi-group Classification Problems via Contextual Coherent Risk Measures}

\author{Darinka Dentcheva \\
        Department of Mathematical Sciences \\
        Stevens Institute of Technology \\
        Hoboken, NJ 07030, USA \\
        \texttt{ddentche@stevens.edu} \\
        \And
        Xiangyu Tian \\
        School of Business \\
        Stevens Institute of Technology \\
        Hoboken, NJ 07030, USA \\
        \texttt{xtian9@stevens.edu}
        }

\date{}

\begin{document}

\maketitle

\begin{abstract}
We propose a new design of fair classifiers for multi-class classification problems in the presence of vector-valued sensitive attributes. In that scenario each sensitive attribute has multiple values and forms several groups relevant to the fairness consideration. Naturally those groups are overlapping and one should also analyze the interaction of factors. Additionally, the decision makers aided by the classification should not violate individual rights at the expense of satisfying fairness metrics at the group level. We propose an approach using the theory and methods of coherent measures of risk aiming at resolving the fairness challenges. Further, we propose a specialized numerical method for solving the resulting optimization problem. The method scales well with the increase of the number of observations. Additionally, we note that the obtained classifier is robust with respect to corrupted data or to situation when data is scarce. We demonstrate the advantages of the proposed framework in comparison to the support-vector machine framework and other methods handling fairness.
\end{abstract}

\keywords{
coherent risk measures, contextual risk, systemic risk, stochastic programming, fairness
}

\section{Introduction}

Fairness has attracted significant attention in the machine-learning community, highlighting concerns about biased outcomes in automated decision-making. The algorithms used in job recruiting, credit risk assessment, and beyond need to be carefully designed to avoid propagating biases.

Different applications motivate different fairness notions. The choice of definition reflects the domain and the harms one aims to prevent. 
We discuss the choice of fainess metric in due course. 

In our paper, we propose a multi-class multi-group method, which directly handles the multi-class classification problem and pays attention regarding the fairness among multiple groups associated with a sensitive attributes including intersectional groups. Furthermore, we can show that our approach results in individual fairness within each group.

Work on algorithmic fairness interventions is often grouped into three approaches, that can be summarized as follows.  
\begin{itemize}
\item Methods modifying the data \cite{feldman2015certifying, kamiran2012data, zemel2013learning}. These data-modification (pre-processing) methods enforce fairness by changing the training data distribution before learning, so that standard classifiers trained on the processed data are less likely to reproduce group disparities. Typical strategies include re-weighting or resampling examples to balance group–label combinations, relabeling a small set of borderline cases to reduce observed discrimination, and “repairing” or learning fair representations that remove (or weaken) information correlated with sensitive attributes while preserving predictive signal. In short, they aim to make the dataset itself less biased, so fairness improvements hold across many downstream models.
\item Methods modifying the learner/objective \cite{rychener2022metrizing, roh2020fairbatch, donini2018empirical, fish2016confidence, wang2024wasserstein}. These methods usually enforce fairness inside the optimization problem itself: they keep the training data largely as-is, but change what the model is asked to minimize (or the way SGD “sees” the data). The approaches include adding fairness constraints, fairness regularizers, or fairness-aware training dynamics to the model or training process.
\item Methods adjusting the predictions \cite{chzhen2020fair, pleiss2017fairness}.  These methods train the model as usual, and then enforce a fairness notion by transforming its outputs (scores/probabilities/labels)—often using the sensitive attribute at prediction time—via group-specific thresholds, score mappings, or controlled randomization, so the final predictions satisfy a desired constraint.
\end{itemize}

While the proposed approached have led to substantial progress, most of them do not discuss the potential issue of data corruption and provide robustness. In practice, however, fairness goals are sometimes pursued in settings where data are noisy, limited, or systematically imperfect: labels may be corrupted, sensitive attributes may be missing or  inaccurately measured, and minority groups may be underrepresented, making constraint satisfaction unstable. This motivates studies such as \cite{9797363,xu2021robust,wang2024wasserstein,fairclass2021celis,taskesen2020distributionally,du2021fair}. These approaches usually aim to incorporate fairness directly into adversarial training, enforce fairness constraints that remain valid under adversarial perturbations, and use distributionally robust optimization to guard fairness guarantees against plausible shifts in the data-generating distribution.

In this paper, our goal is to analyze the classification methods based on coherent measures of risk, which were proposed \cite{dentcheva2025risk}. Our main focus is to enforce fairness in machine learning and provide robustness at the same time via the notions of contexual risk and systemic risk aggregator measures.  We have proposed the concept of contextual risk in \cite{dentcheva2025risk}, where conditional risk are evaluated given contexts. In the context of fairness as associated contexts with the values of the sensitive attributes.  We note that coherent measures of risk come with built-in robustness while remaining amenable to efficient computation as we shall elaborate in due course. We use nonlinear risk aggregation, which enhances robustness against inaccurate data while enabling fairness enforcement. We point out that our method provides a structured way to handle multi-class/multi-group fairness problem. Moreover, our work does not focus on hard-enforcement of fairness such as directly constraining the gap between fairness metrics of different groups. Instead, many coherent measures of risk introduce an implicit penalty on the difference of misclassification risks of each group. The effect is that  each group is treated fairly during the training process, while still focusing on giving accurate predictions even in presence of corrupted data. Furthermore, we show theoretically and numerically that the outcomes of the proposed framework are consistent with many inequality indices known in economics, that are based on widely accepted axioms.  
In this way our approach facilitates enforcement of fairness under complex situations, which is  
confirmed by our numerical experiments. While the extant literature accepts drop of performance for fairness-aware classifiers (see, e.g., \cite{rychener2022metrizing}), we observe high performance of the proposed classifiers for the same or better fairness outcomes.

\section{Coherent measures of risk for groups' and systems' losses }

The risk of a scalar random variable is evaluated by a risk functional, that complies with an accepted set of axioms, which were proposed in \cite{artzner1999coherent} ina special case. Extensions, analysis and further theoretical advances are  presented in \cite{Delbaen,Follmer, RS:2005,RuSh:2006a,PflRom:07, DDARriskbook} and many other works. 
We refer to \cite{DDARriskbook} for an extensive treatment of the theory of risk measures, other risk-averse models, and methods for optimization under uncertainty and risk; see also the references ibid. 

We work with the space of random vectors with realizations in $\Rb^M$, defined on the probability space $(\Omega,\Fc, P)$, which have finite $p$-th moments, $p \in [1, \infty)$, and are indistinguishable on events with zero probability; this space is denoted $\Lc_p(\Omega, \Fc, P;\Rb^M)$. We assume that the random variables represent losses. A counterpart axioms and results exists for random variables representing gains. We focus on the former preferences as we plan to evaluate the risk in classification. 
    
    When we deal with scalar-valued random variables ($M=1$), then a lower semi-continuous functional $\varrho: \Lc_p(\Omega, \Fc, P; \Rb) \to \Rb\cup\{+\infty\}$ is called a \textit{coherent risk measure} if it is convex, positively homogeneous, monotonic with respect to the a.s. comparison of random variables, and satisfies the following translation property:
    \[
        \varrho[Z + a] = \varrho[Z] + a \text{ for all } Z \in \Lc_p(\Omega, \Fc, P),\; a \in \Rb. 
    \]
    If $\varrho[\cdot]$ is monotonic, convex, and satisfies the translation property, it is called a \emph{convex} risk measure. The theory of risk measures for scalar-valued random variables has reached advanced state of development including numerical methods for solving the stochastic optimization problems involving those measures. Less work is associated with high dimensional risks, that is, measures of risk for random vectors, or risk of complex systems with heterogeneous sources of risk.  The need for special attention to random vectors arises in part from the challenge that univariate measures of risk are not additive, unlike the expected value functional.

    We denote the $n$-dimensional vector, whose components are all equal to one by $\onen$, and the random vector with realizations equal to $\onen$ by $\one$. We adopt the following definition, introduced in \cite{almen2024risk,DDARriskbook}: 
    \begin{definition}
    \label{d:riskonvectors}
        A lower semi-continuous functional $\varrho: \Lc_p(\Omega, \Fc, P;\Rb^M) \to \Rb\cup\{+\infty\}$ is a systemic coherent risk measure with preference to small outcomes if it satisfies the following properties:
        \begin{itemize}
            \item[A1.] \emph{Convexity:} For all $Z, Z' \in \Lc_p(\Omega, \Fc, P;\Rb^M)$ and $ \alpha \in (0,1)$, the following holds\\
                $ ~ \varrho[\alpha Z + (1-\alpha)Z'] \leq \alpha \varrho[Z] + (1 - \alpha)\varrho[Z']. $
            \item[A2.] \emph{Monotonicity:} For all $Z, Z' \in \Lc_p(\Omega, \Fc, P;\Rb^M)$, if $Z_i \geq Z'_i$ for all components $i = 1, \dots, M$, $P$-a.s., then $\varrho[Z] \geq \varrho[Z']$.
            \item[A3.] \emph{Positive homogeneity:} For all $Z \in \Lc_p(\Omega, \Fc, P;\Rb^M)$ and all constants $t > 0$, the relation $\varrho[tX] = t\varrho[X]$ holds.
            \item[A4.] \emph{Translation equivariance:} For all $Z \in \Lc_p(\Omega, \Fc, P;\Rb^M)$ and for all constants $a \in \Rb$, the following relation holds $\varrho[Z + a\one] = \varrho[Z] + a\varrho[\one].$
        \end{itemize}
    \end{definition}
 It is shown in \cite{almen2024risk,DDARriskbook} that any systemic risk measure $\varrho$, which is proper, lower semicontinuous, and satisfies those axioms, can be represented as follows:
    \begin{equation}
        \varrho[Z] = \sup_{\zeta \in \Ac_\varrho} \langle \zeta, Z \rangle.
        \label{general dual}
    \end{equation}
    The set $\Ac_\varrho\subset \Lc_q(\Omega, \Fc, P;\Rb^M)$ with $\frac{1}{p} + \frac{1}{q} =1$ is defined as:
    \[
        \Ac_\varrho = \Big\{ \zeta \in \Lc_q(\Omega, \Fc, P;\Rb^M) ~|~ \int_\Omega \zeta(\omega) dP(\omega) = \mu_\zeta, ~ \zeta \geq 0 \text{ a.s.},\;  \langle \one, \mu_\zeta \rangle =r \Big\},
    \]
    where $r\in\Rb$ is a constant. The set $\Ac_\varrho$ is the convex subdifferential $\partial \varrho(0)$ of the risk measure. A systemic measure of risk $\varrho$ is normalized if $\varrho[\one] = 1$, in which case,  $r =1$ for all $\zeta \in \Ac_\varrho$.
    In the special case when $N=1$, we obtain the widely used dual representation of coherent measures of risk for scalar-valued random variables 
    \begin{equation}
        \varrho[Z] = \sup_{\frac{dQ}{dP} \in \Ac_\varrho} \Eb_Q [ Z ],
        \label{eq:scalar dual}
    \end{equation}
    where $\frac{dQ}{dP}$ is the Radon-Nikodym derivative of the measure $Q$ with respect to the reference measure $P$. The dual representation \eqref{general dual} shows the link between minimizing a coherent measure $\varrho(Z)$ to a distributionally robust optimization (DRO) problem.  The DRO problem with a loss function $L(Z(\theta, \omega))$  and an ambiguity set $\Ac \subset\Pc (\Omega)$ is formulated as 
    \begin{equation}
    \min_{\theta}\ \sup_{Q\in\Ac}\ \Eb_Q \big[L (Z(\theta,\omega))\big].
    \label{eq:DRO}
    \end{equation}
    A widely used coherent risk measure for scalar-valued random variables is the well-known Mean - Upper Semideviation, which is defined as follows
    \[
    \varrho(Z) = \Eb[Z]+\kappa\big\|(Z-\mathbb{E}[Z])_{+}\big\|_p,\; \text{ where  } \kappa\in[0,1].
    \]
    Here $\|\cdot\|_p$ stands for the $p$-norm in $\Lc_p(\Omega, \Fc, P;\Rb)$  and $(a)_+ = \max(0,a)$. 
    The Mean - Upper Semideviation has the dual representation \eqref{eq:scalar dual} with the dual set 
    \[
    \Ac_\varrho =\left\{Q\ll P:\;\tfrac{dQ}{dP}=\one+\zeta-\mathbb{E}[\zeta]\one,\; \|\zeta\|_q \leq \kappa,\; \zeta \geq 0\right\}, 
    \]
    i.e, $Q$ are probability measures with densities w.r.to $P$ given by $\one+\zeta-\mathbb{E}[\zeta]\one.$
    Notice that the parameter $\kappa$ controls the risk-averse level of the objective; larger value of $\kappa$ allows larger weight-differences in $\zeta$, leading to a larger ambiguity set.

Consider a finite probability space $(\Omega_M,\Fc_M, \pi)$, where $\Omega_M = \{1,\dots,M\}$, $\pi$ is a probability mass function, and $\Fc_M$ contains all subsets of $\Omega_M$. Given a collection of $M$ measures of risk $\varrho_i:\Lc_p(\Omega,\Fc, P;\Rb)\to\Rb$, $i\in\Omega_M $,  we associate with a random $M$-dimensional vector $Z$ defined on $(\Omega,\Fc, P)$, a random variable $R_Z$ on the space $\Omega_M$ as follows. The realizations of $R_Z$ are 
    \begin{equation*}
        \label{R_Z}
        R_Z(i) = \varrho_i(Z^i),\quad i\in\Omega_M. 
    \end{equation*}
 Let a coherent measure of risk $\varrho_0:\Lc_\infty(\Omega_M,\Fc_M, \pi)\to\Rb$  be chosen as an aggregator. We define the measure of systemic risk 
     $\varrho_{\rm sys}: \Lc_p(\Omega, \Fc, P; \Rb^M)\to\Rb$ as follows:
    \begin{equation}
         \varrho_{\rm sys} (Z) = \varrho_0(R_Z). 
         \label{syst_risk2} 
    \end{equation} 
This type of measure satisfies the axioms in Definition~\ref{d:riskonvectors} as shown in \cite{almen2024risk}.

We give two examples of systemic risk-measures that will become useful in our further developments.

We may choose as $\varrho_0$ to be the mean-upper-semideviation risk measure of order $q$ ($q\geq 1$) and evaluate all components of the error vector $Z$ by the same law-invariant coherent measure of risk $\varrho(\cdot)$. The description of the total risk evaluation is the following: 
    \begin{equation}
    \label{d:outer-MSD-p}
  \varrho_{\rm sys}(Z)  = \sum_{i\in\Omega_M } \pi_i \varrho[Z^i] + \kappa \bigg(\sum_{i\in\Omega_M } \pi_i \Big(\varrho(Z^i) - \sum_{i\in\Omega_M } \pi_j \varrho(Z^j) \Big)_+^q\bigg)^\frac{1}{q}      
    \end{equation} 
    with $\kappa\in[0,1]$.
 When $q=1$, the expression simplifies as follows
 \begin{equation}
 \label{d:outer-MSD-1}
 \begin{aligned}
\varrho_{\rm sys}^1(Z) &= \sum_{i\in\Omega_M }  \pi_i\bigg( \varrho[Z^i] + \kappa \sum_{j=1}^M \pi_j\big(\varrho[Z^i] - \varrho[Z^j]\big)_+ \bigg)\\
& = \sum_{i\in\Omega_M }  \pi_i \varrho[Z^i] + \kappa  \sum_{\substack{i,j=1\\j\neq i}}^M \pi_i\pi_j\big|\varrho[Z^i] - \varrho[Z^j]\big| 
     \end{aligned}
       \end{equation}
    This representation shows that the individual risks of the components are aggregated with an additional penalty on the deviation of the individual risks from that average risk. This property is crucial to our treatment of fairness in classification. 

    Another example would be to aggregate the individual risks  by a convex combination of the expected value and the Average Value-at-Risk at some level $\alpha$; again all components of the error vector $Z$ may be evaluated by the same law-invariant coherent measure of risk $\varrho(\cdot)$.  
    For any $\kappa \in [0,1]$, this systemic measure of risk takes on the form:
    \begin{equation}
    \label{e:avar-comb}
        \varrho_\sys^2 (Z)  = (1-\kappa) \sum_{i\in\Omega_M} c_i \varrho (Z^i) + 
                          \kappa \inf_{\eta \in \Rb} \Big\{  \eta + \frac{1}{\alpha} \sum_{i\in \Omega_M} c_i\big(\varrho (Z^i) - \eta\big)_+ \Big\}.
    \end{equation}
    Here, the infimum with respect to $\eta \in \Rb$ is taken over the individual risks of the components $\varrho (Z^i)$, $i\in \Omega_M$. Hence, this method of aggregation imposes additional penalty for the components whose risk exceeds a certain threshold, which can be shown to be the $\alpha$ -quantile of the distribution of $R_Z$.

    We note that the linear aggregation is a special case of a systemic risk measure because one can choose  $\varrho_0(R_Z) = \Eb_\pi(R_Z) = \sum_{i\in\Omega_M } \pi_i \varrho(Z^i)$. The expectation is the simplest coherent measure of risk. 
    In each case, we could also use different risk measure for each component $Z^i$, should we have a reason to do so.

\section{Fairness notions and challenges} 
\label{sec:evaluating_fairness}

In order to discuss fairness, one needs to identify a definition agreeable for the application it is discussed.

Naturally, the fairness definitions and the associated algorithmic ideas are first developed in a binary classification scenario, as it answers most pertinent questions such as will a given loan be re-paid; is a project likely to success; is a treatment effective for the individual, etc. 

For a binary decision,  we consider fair-sensitive feature with values $\mathcal{S}={1,\dots, S}$, which is present in the data points to be classified. The labels are $Y=1$ (favorable) and $Y=0$ (unfavorablle).  
The notion of \emph{demographic parity} requires that the classifier outputs a decision $\hat{Y}$ such that the probability 
$P(\hat{Y}=1 | G=s) $ is the same across $s\in\Sc$. 
Here, $P$ is the frequency within the given dataset used as an estimation of the true conditional probability. 

Another notion is so called \emph{equalized odds}, which requires that the true positive rates and false positive rates are both equal, i.e. for all $s_1,s_2\in\Sc$, the following relations are satisfied.
\begin{gather*}
P(\hat{Y}=1 | Y=1, G=s_1) = P(\hat{Y}=1 | Y=1, G=s_2)\\
  P(\hat{Y}=1 | Y=0, G=s_1)= P(\hat{Y}=1 | Y=0, G=s_2).  
\end{gather*}
The notion of \emph{equal opportunity} requires only the first equality, regarding the true positive rates (TPR). 

The multi-class setting is less investigated and the binary situation cannot be carried to it in a straightforward manner, as there are no true positive/true negative predictions (or false positive/false negative). The two standard approaches are to define fairness per-class (one-vs-rest, giving a rate for each class $\times$ group and then aggregate. In some work the maximum difference across the various groups is taken and then averaged across classes, in others an average difference is calculated. In this way the fairness metric reduces to the classical one when there are two classes and two groups. In \cite{denis2024fairness}, the authors extend the definition of approximate fairness based on demographic parity to multi-class classification with two groups, i.e. $S=2$. They define an index of unfairness as the maximum difference between $P(\hat{Y}= k | Y=1, G=s_1)$ for $k$ varying over the classes. 
The authors impose demographic-parity constraints and compute multi-class thresholds from the Lagrange multipliers of them.  

In \cite{alghamdi2022beyond}, the authors modify the notion of demographic parity, equalized odds, and overall accuracy equality into a multi-class, multiple-protected-group setting by post-processing method based on information projection. 
The most unified treatment of multi-class multi-group setting as far as we could see so far is based on optimal transport. It is established that demographic-parity fair prediction is equivalent to a certain Wasserstein-barycenter problem (\cite{chzhen2020fair,jiang2020wasserstein}, see also \cite{hu2023fairness}). 
One phenomenon to be aware of when dealing with multiple groups is described in the paper \cite{kearns2018preventing}. It illustrates that treating one sensitive attribute at a time might lead to a model fair on each attribute but maximally unfair across intersectional groups. The authors propose auditing algorithms for subgroup fairness over combinatorially many subgroups defined by false-positive, false-negative, and classification rates. 

The individual fairness is a recurring critique of group-fairness criteria because a model can satisfy them at the group level while treating individuals \emph{within}
a group inconsistently, arbitrarily, or even worse than before the fairness constraint was imposed.  In \cite{dwork2012fairness}, the authors introduce individual
fairness --- ``treat similar individuals similarly,'' formalized as a Lipschitz condition on a task-specific similarity metric. Their motivating critique of demographic parity is that it constrains only group statistics and so permits within-group abuse: parity can be met by selecting the qualified members of one group and the unqualified members of
another, leaving group rates equal but individuals treated incoherently. In \cite{long2023individual} it is shown that fairness interventions optimized solely
for group fairness and accuracy can \emph{exacerbate} predictive multiplicity: the existence of many near-equivalent models that disagree on individual
predictions.  The authors find that some sort of arbitrariness is present in the predictions that could flip on retraining with a different random seed. They argue that ``arbitrariness'' should be considered alongside accuracy and group fairness.

In \cite{kim2018computationally} the authors propose a \emph{metric multifairness} framework to treat subpopulations similarly, where the subpopulations
range over a rich, possibly overlapping class of comparison sets. This is the
individual-fairness analogue of the multicalibration idea presented in \cite{hebertjohnson2018multicalibration}, which asks for
calibration to hold not just marginally per group but simultaneously across a rich
family of overlapping subgroups; both are attempts to push a guarantee down
from groups toward the individual level. 


Measures for within-group unfairness using the
inequality-index machinery known in economics are advocated in \cite{speicher2018unified}, where 
a model's unfairness is quantified with a generalized-entropy index over a per-individual
``benefit''. This index has the benefit of additivity, in which the overall individual-level unfairness decomposes
into a \emph{between-group} and a \emph{within-group} component. Let us emphasize that classical inequality measures in exonomics such as the Gini coefficient are based on axiomatic foundations. The work \cite{speicher2018unified} abuilds on it. The paper \cite{williamson2019fairness} proved that every fairness risk measure satisfying natural axioms is a coherent risk measure, establishing a connection to the risk measures used in our training formulation. Despite this progress, most of the fairness literature still reports point estimates of ad hoc metrics without statistical inference.

Finally, there is a discussion about the compatibility of individual and group fairness. 
In \cite{xu2024incompatibility}, the authors 
establish conditions under which optimal statistical-parity post-processing
is compatible with  the Lipschitz condition for individual fairness, introduced by \cite{dwork2012fairness} 
and, when the two conflict, identify the regions of the error--disparity Pareto
frontier that still satisfy the individual-fairness requirement.   



We start with a general definition. Consider a classifier $\psi$ applied to data with true labels $Y$ taking values in $\{1, \dots, M\}$ and a sensitive attribute $G$ taking values in $\Sc = \{1, \dots, S\}$. For each data point $X_i$, the classifier produces a prediction $\psi(X_i)$. The prediction is a random variable whose distribution may depend on both the true label and the group membership. We say that the classifier $\psi$ is \emph{fair} with respect to the sensitive attribute $G$ if, conditioned on the true label, the distribution of the predictions does not depend on the group. That is, for each point $(Y,X)$, the classifier $\hat{Y} = \phi(X)$ satisfies equality \eqref{eq:fairness_def} for all classes $k, c\in\{1,\dots,M\}$ and for all group attributes $s\in\Sc$: 
\begin{equation}
\label{eq:fairness_def}
P[\hat{Y} = k \mid Y = c,\, G = s] = P[\hat{Y} = k \mid Y = c] \quad .
\end{equation}
In words, knowing which group a data point belongs to provides no additional information about the prediction it receives, beyond what is already known from its true label. This requires not only that each group achieves the same per-class accuracy, but that the entire pattern of correct and incorrect classifications be independent of group membership. This definition generalizes equalized odds from binary to multi-class setting. In comparison,  \cite{alghamdi2022beyond} consider the true class $Y$ and the group $s\in\Sc$ and use a probabilistic classifier. They require that the entire per-group conditional confusion distribution match the group-marginal one by using a tolerance level $\alpha$ and setting:
\[
     \Big| \frac{P(\hat{Y} = k \mid Y = c,\, G = s)}{P(\hat{Y} = k \mid Y = c)} -1 \Big|\leq \alpha. 
\]
So for each true class $c$ and each predicted class $k$ the probability of predicting $k$ when the truth is $c$ must be nearly the same in a group $s$ as in the population. When the task is binary, this becomes equality of the false-positive and false-negative. 

The definition \eqref{eq:fairness_def} implies equal opportunity as a special case (take $c = 1$ and $k = 1$), but it additionally requires the misclassification patterns to be group-independent. Two groups could satisfy equal opportunity while having very different error distributions across the remaining classes while under \eqref{eq:fairness_def}, this would still constitute unfairness. We note that in binary classification with $M = 2$, the prediction distribution conditioned on $Y = 1$ is fully determined by the TPR, so equal opportunity and \eqref{eq:fairness_def} conditioned on $Y = 1$ are equivalent.

Frequently, when the sensitive attribute takes $S > 2$ values (e.g., race, age, education, country, etc), the Equal Opportunity is typically computed as a ratio of the smallest True Positive Rate (TPR) over the largest one. This uses information from only the two most extreme groups and ignores what happens in between. Two classifiers could have the same EO-ratio but very different treatment of the intermediate groups. In a $M$-class setting, the issue is more fundamental: equal recall for each class does not imply \eqref{eq:fairness_def}, since two groups could have the same recall but very different distributions of errors across the other classes. The standard metrics, which summarize each group by a single number per class, cannot detect such differences. Most importantly, these metrics are reported as point estimates without uncertainty quantification. An EO-ratio of $0.85$ might indicate a serious fairness violation if computed on $10000$ samples, or it might be entirely consistent with a perfectly fair classifier if computed on $50$ samples. The work \cite{cherian2023statistical} provides a comprehensive discussion of this gap, framing fairness auditing as a multiple hypothesis testing problem and showing that bootstrap-based inference can provide simultaneous confidence bounds across many subgroups.

We note also the work of \cite{rychener2022metrizing}, where it is proposed to measure unfairness via integral probability metrics, connecting fairness to the Wasserstein distance between group-conditional outcome distributions.  

In addition to the equalized odds and equality ratios, we use the classical Gini index of inequality. Furthermore, 
we propose to evaluate the validity of the equalities in \eqref{eq:fairness_def} by the classical statistical test of goodness of fit. The key observation is that the definition requires the conditional prediction distribution be the same across group. This is precisely the null hypothesis tested by Pearson's chi-square test of homogeneity. Fix a true class $c \in \{1, \dots, M\}$ and consider the test samples with true label $c$. Each sample belongs to one of $S$ groups and receives one of $M$ possible predictions, giving a $M \times S$ contingency table where cell $(k, s)$ contains the count $O_{ks}$ of samples from group $s$ that received prediction $k$. Under the null hypothesis the classifier map $\psi(\cdot)$ satisfies:
\[
H_0: \quad P[\psi(X_i) = k \mid Y_i = c,\, G_i = s] = P[\psi(X_i) = k \mid Y_i = c] \quad \text{for all } k, s,
\]
the expected count in cell $(k, s)$ is
\[
E_{ks} = \frac{R_k \cdot n_s}{n}, 
\]
where $R_k = \sum_s O_{ks}$ is the total number of samples predicted as class $k$, $n_s = \sum_k O_{ks}$ is the total number of samples in group $s$, and $N = \sum_{k,s} O_{ks}$ is the grand total. This formula follows from the definition of statistical independence: under $H_0$, the probability of falling into cell $(k,s)$ equals the product $P[\psi(X_i) = k \mid Y_i = c] \cdot P[G_i = s \mid Y_i = c]$. The Pearson chi-square statistic
\begin{equation}
\label{eq:chi2_fairness}
\chi^2 = \sum_{k=1}^{M} \sum_{s=1}^{S} \frac{(O_{ks} - E_{ks})^2}{E_{ks}}
\end{equation}
follows a $\chi^2\big((M-1)(S-1)\big)$ distribution asymptotically under $H_0$. A large value indicates that the prediction distributions differ significantly across groups, providing statistical evidence of unfairness. 

For a multi-class problem, we construct one contingency table per true class and obtain $M$ chi-square statistics with corresponding $p$-values. To control the family-wise error rate, we apply the Bonferroni correction: the classifier is declared fair at level $\alpha$ if all $M$ $p$-values exceed $\alpha / M$. That being said, if we are defining the fairness in a statistical parity style, where we do not care about a sample's true label, but only care about the probability of its prediction, the test would be even easier since we only need one test overall instead of one test per class.

\section{Contextual risk and fairness} 
\label{sec:contextual_risk_and_fairness}

We proposed risk-averse formulation based on contextual risk and aggregation of risk by outer risk measure with specific properties. For example, the mean-upper semi-deviation risk measure used as aggregator suggests that fairness within the learning model between the classes is forced. This can be categorized into the modifying learner/objective methods we discussed in the literature review. We have seen models that directly modify the loss function, such as \cite{rychener2022metrizing}, which uses a penalty term on the distance between the distributions of predictions within the two groups. Other works use extra constraints to force fairness, e.g. \cite{donini2018empirical}. Our two-stage model naturally involves an implicit penalty term within the formulation of the outer risk measure, If we use the mean-semideviation risk measure as aggregator, then it penalizes the deviation of the risk for each class from the average risk. When the Average Value-at-Risk at level $\alpha\in(0,1)$ is involved in the objective function, then  the deviation of the risk for each class from the $\alpha$-quantile of the risk distribution to classes is penalized.  In this way there is some similarity to \cite{rychener2022metrizing}. However, one issue arises: when considering sensitive groups or attributes, we are indeed assessing the risk with conditions; therefore, the behavior and properties of such conditional risks need to be well-defined.

We consider the scenario in which a fair-sensitive categorical feature with values $\mathcal{S}={1,\dots, S}$ is present in the data points to be classified. To make the prediction fair with respect to the groups within all classes, we introduce a contextual risk-measure $\varrho_c[Z]$. Assume that $Z=f(\vartheta,X,Y,G)$, where $\vartheta$ determines the classifier, $X$ stands for a random data point and $Y$ stands for the class label and $G$ indicates the value of said feature. Let the function $f:\Rb^{2n}\times\Jc\times \mathcal{S}\to \Rb^M$ be convex and monotonically non-decreasing with respect to the second argument for any value of the other arguments. We need to consider the classification error $Z^i$ in every context $s\in\mathcal{S}$ for all classes $i\in\Jc$.
Consider the probability space $(\Omega,\Fc,P_{XY})$ where $(X,Y)$ are defined. The measure $P_{XY}$ is disintegrated into the marginal distribution $\pi$ of $Y$ and the conditionals $P_i$ for $X| Y=i$. We consider the spaces 
\[
\mathcal{Y}_{i,s}=\big(\Omega,\Fc\cap \{ G=s\}, P_{is}\big),
\]  
where $P_{is}$ is the $P_i$-induced probability measure on $\Fc\cap \{ G=s\}.$
A contextual risk measure is a set of conditional coherent measures of risk evaluating the risk given a context. An aggregation measure provids the total risk for all contexts. 
More precisely, we fix a real-valued coherent risk measures $\rho_{i,s} [Z|Y=i, G=s]$ for each $i\in\Jc$ and $s\in\mathcal{S}$ defined on $\mathcal{Y}_{i,s}$ and choose 
$\varrho_c[Z^i]$ as an aggregation measure to obtain the total risk of each class over all contexts. Additional aggregation using an outer risk measure $\varrho_0[Z]$ will aggregate the risk over all the classes. We denote the distribution of all classes in the entire population by $\pi$, i.e., $\pi$ is a vector with $M$ positive components such that $\sum_{i=1}^M\pi_i = 1.$

\begin{proposition}
  Let  $\rho_{i,s} [\cdot]$ be coherent measures of risk for  all  $i\in\Jc$ and all $s\in\mathcal{S}$, 
risk measures $\varrho_{i}: \Lc_\infty(\mathcal{S},\Fc_S,  P_i|G=s)\to\Rb$ be coherent measures of risk for all $i\in\Jc$, and 
$\varrho_0:\Lc_\infty(\Omega_M,\Fc_M, \pi)\to\Rb$ be coherent as well. 
Denote
\begin{equation*}
	    V_c(i,s) = \rho_{i,s}[Z^i|G=s] \quad \text{and} \quad W(i) = \varrho_{i}[V_c(i,\cdot)],\quad s\in\mathcal{S},\; i\in\Jc.
\end{equation*}
 Then the risk measure $\varrho_\sys [Z]= \varrho_0 [W]$ satisfies axioms A1-A4 of systemic measure of risk. 
\end{proposition}
\begin{proof}
(i) Given any $Z_1$, $Z_2$ and $\alpha \in (0,1)$, we consider the random vector $Z_3 = \alpha Z_1 + (1-\alpha)Z_2$.  It follows that
\begin{multline*}
V_c^3(i,s) = \rho_{i,s}[Z^i_3|Y=s] \leq \alpha\rho_{i,s}[Z^i_1|Y=s] + (1-\alpha)\rho_{i,s}[Z^i_2|Y=s] \\
= \alpha V_c^1(i,s)+ (1-\alpha) V_c^2(i,s).
\end{multline*}
by the convexity of $\rho_{i,s}[\cdot]$ for all $i\in\Jc$ and all $s\in\mathcal{S}$.  Hence
\[
W^3(i) =  \varrho_{i}[V_c^3 (i,\cdot)] \leq \alpha \varrho_i[V_c^1(i,\cdot)]+ (1-\alpha) \varrho_i[V_c^2(i,\cdot)] =
\alpha W^1(i)+ (1-\alpha) W^2(i).
\]
by the same arguments. Analogously,  
\[
\varrho_\sys[Z_3]= \varrho_0 [W^3]\leq \alpha \varrho_0 [W^1]+ (1-\alpha) \varrho_0 [W^2] =
\alpha \varrho_\sys[Z_1] +(1-\alpha) \varrho_\sys[Z_2], 
\] 
which establishes the convexity property. 

(ii) Suppose the vectors $Z_1, Z_2 $ satisfy $Z_1\leq Z_2$ a.s. This implies that $Z_1^i \leq Z_2^i$ a.s. Using the monotonicity of the risk measures $\rho_{i,s}[\cdot]$, $\varrho_i[\cdot]$ and $\varrho_0[\cdot]$, we infer the monotonicity of $\varrho_\sys[\cdot].$  

 (iii) The positive homogeneity follows in a straightforward manner from the definition. 

	    (iv) Given a random vector $Z$ and a real constant $a$, we calculate the risk $V^+_c(i,s)$ of the translated error in context $s$  as follows:
     \[
     V^+_c(i,s) = \rho_{i,s}[Z^i +a|G=s] = \rho_{i,s}[Z^i|G=s] +a = V_c(i,s) + a \quad \text{for all } s\in\mathcal{S}, \; i\in\Jc.
     \]
    The second equality holds by the translation property of the coherent measures of risk. This implies that
   $W^+(i) =  \varrho_{i}[V_c^+ (i,\cdot)] = \varrho_{i}[V_c (i,\cdot) + a] = W(i) +a $ by the same argument. Hence,  
	    $\varrho_0[W+a]= \varrho_0[W] + a,$ concluding that property (A4) holds as well.
\end{proof}  

This structure of measuring risk allows us to enforce fairness within each class by choosing the aggregate measures $\varrho_i$ to be such that deviation from the average risk or excessive risk above a certain quantile is penalized.  

As a class-level aggregator $\varrho_c$, we use the mean-semideviation risk measure. If applied to 
two classes: $A$ and $B$, with one sensitive context valued in $\mathcal{S} =\{1,2\}$, it leads to the representation:
\[
\varrho_c = \pi_{c,1}\varrho_{c|1}+\pi_{c,2}\varrho_{c|2}+ \kappa_{\rm mid}\, \pi_{c,1}\pi_{c,2}\,|\varrho_{c|1}-\varrho_{c|2}|,\quad c=A,B.
\]
Here $p_{c,s}$ is the conditional weight of group $s$ within class $c$ and $\kappa_{\rm mid}\in[0,1]$. This middle-level aggregation penalizes within-class group disparity: when the risks $\varrho_{c|1}$ and $\varrho_{c|2}$ differ, the semideviation term increases the total class risk, discouraging the classifier from achieving low risk for one group at the expense of the other.

For the outer aggregation $\varrho_0$, one could simply use the expected value $\varrho_0=\pi_A\varrho_A+\pi_B\varrho_B$, which balances the group-regularized class risks. However, a linear outer aggregation does not penalize imbalance between the class-level risks --- a classifier that achieves low risk for one class but high risk for the other is not penalized beyond the weighted average. To address this, we propose using a non-linear aggregation, e.g, the mean semideviation, at the outer level as well:
\[
\varrho_0 = \pi_A W(A) + \pi_B W(B) + \kappa_{\rm out}\, \pi_A \pi_B\, |W(A) - W(B)|,
\]
with $\kappa_{\rm out}\in[0,1]$. This gives a fully nonlinear three-level risk structure: sample losses within each scenario are aggregated by the inner risk $\rho_{i,s}$, scenario risks within each class are aggregated by the middle MSD with parameter $\kappa_{\rm mid}$, and class risks are aggregated by the outer MSD with parameter $\kappa_{\rm out}$. By the Proposition above, this three-level composition remains a coherent systemic risk measure.

\section{Three-stage regularized decomposition}
\label{sec:three_level_decomposition}

We now describe a regularized decomposition method for soling the three-level nonlinear risk structure introduced above. Recall that all functions and risk measures involved depend on the parameter $\vartheta$ of the classifier.  The overall objective is
\begin{gather*}
\varrho_\sys(\vartheta) = \varrho_0\big[W(1;\vartheta),\dots,W(M;\vartheta)\big], \\ 
W(i;\vartheta) = \varrho_i\big[V_c(i,1; \vartheta),\dots,V_c(i,S; \vartheta)\big], \quad V_c(i,s; \vartheta) = \rho_{i,s}[Z^i(\vartheta)|G=s].
\end{gather*}
For the outer risk $\varrho_0$ and for the class aggregators $\varrho_i [(i,\cdot,\vartheta)]$, the dual representations are 
\begin{align}
\varrho_0[W;\vartheta] & = \max_{\mu \in \Ac_0} \sum_{i=1}^{M} \mu_i\, W(i;\vartheta).
\label{eq:outer-dual}\\
\varrho_i[V(i,\cdot;\vartheta)] & = \max_{\zeta^i \in \Ac_i} \sum_{s \in \Sc} \zeta_s^i\, V(i,s;\vartheta),
\label{eq:class-dual}
\end{align}
The problem with non-linear aggregation takes on the form:
\begin{align}
\min_{\vartheta,\alpha} \; & \alpha + \sigma \sum_{i\in\Jc} \left\|v^i\right\|^2\quad 
 \text{s.t }\; \alpha \geq \sum_{i=1}^{M} \mu_i\, W(i;\vartheta)\quad \forall \mu \in \Ac_0;
\label{p:fairness-Istage}\\
W(i,\vartheta) = \min_{q_i} \; & q_i \quad \text{s.t }\;
  q_i \geq  \sum_{s \in \Sc} \zeta_s^i\, V(i,s;\vartheta)\quad \forall \zeta^i \in \Ac_i
  \label{p:fairness-IIstage}
\end{align}
\begin{equation}
\label{p:fairness-IIIstage}
\begin{aligned}
   V(i,s,\vartheta) = \min_{Z^i} \quad & \varrho_{i,s}[Z^i|G=s]\\
    \text{s.t.} \quad & z^i_\ell \geq \psi(x^i_\ell;\vartheta_i) - \psi(x^i_\ell;\vartheta_j) +1 \quad 
     j \in\Jc^{-i}, \;  \ell \in\{1, \ldots, m_i: G = i \}\\
    & Z^i \geq 0.
\end{aligned}
\end{equation}


In the two-stage problem treated in \cite{DentchevaTian}, the outer aggregation $\varrho_0$ is applied directly to the class risks $R_i(\vartheta)$, and the master problem maintained two sets of cuts: objective cuts approximating $\varrho_0$ and scenario cuts approximating each $R_i(\cdot)$. With the contextual risk structure, we have an additional middle level: the class risks $W(i;\vartheta)$ are themselves nonlinear aggregations via $\varrho_i$ of the contextual risks $V_c(i,s;\vartheta)$. 

The construction of cuts at each level follows from the dual representation. 
At each iteration $k$, we evaluate the class risks $W^k(i;\vartheta)$ and identify the maximizer $\mu^k \in \Ac_0$ for the random variable $W^k(\cdot;\vartheta)$ that has
realizations $W^k(i;\vartheta)$ with probabilities $\frac{m_i}{N}$, $N = \sum_{i=1}^M m_i$ . This provides the outer cut
\begin{equation}
\alpha \geq \sum_{i=1}^{M} \mu_i^k\, q_i,
\label{eq:outer-cut}
\end{equation}
where $q_i$ are variables approximating the class risks. 
This is the same type of cut as the objective cut in \cite{DentchevaTian}, but now acting on the class-level variables $q_i$ rather than the scenario-level variables. Similarly, for each class-level aggregator $\varrho_i$, at each iteration $k$, we evaluate the scenario risks $R_{is}^k = V(i,s:\vartheta)$ for $s \in \Sc$ and identify the maximizer $\zeta^{i,k} \in \Ac_i$. This provides the class-level cut
\begin{equation}
q_i \geq \sum_{s \in \Sc} \zeta_s^{i,k}\, r_{is},
\label{eq:class-cut}
\end{equation}
where $r_{is}$ are variables approximating the contextual risks. These cuts are the new ingredient compared to the method in \cite{DentchevaTian}: they approximate each $\varrho_i$ from below as a function of the scenario risks.

At the bottom level, the scenario cuts are constructed exactly as in \cite{DentchevaTian}. For each scenario $(i,s)$, we solve the subproblem \eqref{p:fairness-IIIstage} to obtain the scenario risk $R_{is}^k$ and the subgradient $g_{is}^k$ of $V(i,s;\vartheta)$ with respect to $\vartheta$, which provides the cut
\begin{equation}
r_{is} \geq R_{is}^k + \langle g_{is}^k, \vartheta - \vartheta^k\rangle.
\label{eq:scenario-cut}
\end{equation}

The regularized master problem at iteration $k$ takes the form:
\begin{equation}
\label{eq:master_1level}
\begin{aligned}
\min_{\vartheta,\alpha} \quad & \alpha + \sigma \|v\|^2 + \beta\|\vartheta - w^k\|^2 \\
\text{s.t.} \quad & \alpha \geq \textstyle\sum_{i=1}^{M} \mu_i^{\kc}\,\bar{q}_i, && \kc \in \Kc_0, \\
& \alpha\geq 0.
\end{aligned}
\end{equation}
For each $i\in\Jc$, we consider the problem
\begin{equation}
\label{eq:master_2level}
\begin{aligned}
\bar{q}_i = \min_{r,\, q,\, \alpha} \quad & q_i \\
\text{s.t.} \quad &  q_i \geq \textstyle\sum_{s \in \Sc} \zeta_{s}^{i,\kc}\, r_{is}, && \kc \in \Kc_i,\; \\
& r_{is} \geq R_{is}^{\kc} + \langle g_{is}^{\kc}, \vartheta - \vartheta^{\kc}\rangle, && \kc \in \Kc_{is},\; s\in\Sc, \\
& r_{is}, q_{i} \geq 0\quad s\in\Sc.
\end{aligned}
\end{equation}
Here $\sigma > 0$ is the regularization parameter, $\beta > 0$ is the proximal parameter, and $w^k$ is the current stability center. The sets $\Kc_0$, $\Kc_i$, and $\Kc_{is}$ collect the indices of the outer, class-level, and scenario cuts, respectively.

We propose the following three-level regularized multi-cut method with a parameter $\delta \in (0,1)$.

   \begin{itemize}
    \item[]   \vspace{0.2cm} \hrule \vspace{0.2cm}
                \textbf{Three-stage Classification with Contextual Systemic Risk} 
                \vspace{0.2cm} \hrule \vspace{0.2cm}  
\item[\emph{Step 0.}] Set $k=1$. Choose initial decision variable $\vartheta^1$.
\item[\emph{Step 1.}] For each  $(s,i)$, $s\in\Sc;\, i \in \Jc$, solve the third stage subproblem \eqref{p:fairness-IIIstage}. Denote its optimal value by $R_{is}^k$ and calculate the subgradient $g_{is}^k$. Add the new scenario cuts to $\Kc_{is}$.
\item[\emph{Step 2.}] For each class $i \in\Jc$, evaluate the class-level risk $W^k(i) = \varrho_i[R_{is}^k : s \in \Sc]$ by solving \eqref{eq:master_2level} and calculate $\zeta^{i,k}$ using \eqref{eq:class-dual}. Add the new class-level cuts to $\Kc_i$, $i\in\Jc$.
\item[\emph{Step 3.}] Calculate the systemic risk $\varrho_\sys^k = \varrho_0[W^k(\cdot)]$  and calculate $\mu^k$ using \eqref{eq:outer-dual}. Add the new outer cut to $\Kc_0$.
\item[\emph{Step 4.}] Determine the new center $w^k$ as follows. If $k=1$ or 
\[
\varrho_\sys^k \leq (1-\delta)\bar{\varrho}^{k-1} + \delta \alpha^{k-1},
\]
then set $w^k = \vartheta^k$ and $\bar{\varrho}^k = \varrho_\sys^k$ (descent step). Otherwise, set $w^k = w^{k-1}$ and $\bar{\varrho}^k = \bar{\varrho}^{k-1}$ (null step).
\item[\emph{Step 5.}] Solve the master problem \eqref{eq:master_1level}. Denote the solution by $\alpha^k, \vartheta^k$.
\item[\emph{Step 6.}] If $\bar{\varrho}^k = \alpha^k$, then stop ($w^k$ is an optimal solution); otherwise continue.
\item[\emph{Step 7.}] Remove from the sets $\Kc_0$, $\Kc_i$, and $\Kc_{is}$ the constraints whose Lagrange multipliers are 0 at the solution of \eqref{eq:master_2level} and 
\eqref{eq:master_1level}. Increase $k$ by 1 and go to Step~1.
  \item[]   \vspace{0.2cm} \hrule \vspace{0.2cm}
\end{itemize}

\medskip

We observe that the second stage problem only calculates the best approximation of the risk of individual classes and does not change the classification parameter. 
Using the monotonicity property of the risk measures involved and the preserving the hierarchical structure, we may collapse the first and the second stage into one.
In that case, we accumulate all three sets of cuts in one master problem. The regularized master problem at iteration $k$ takes the form:
\begin{equation}
\label{eq:master_3level}
\begin{aligned}
\min_{\vartheta,\, r,\, q,\, \alpha} \quad & \alpha + \sigma \|v\|^2 + \beta\|\vartheta - w^k\|^2 \\
\text{s.t.} \quad & \alpha \geq \textstyle\sum_{i=1}^{N} \mu_i^{\kc}\, q_i, && \kc \in \Kc_0, \\
& q_i \geq \textstyle\sum_{s \in \Sc} \zeta_s^{i,\kc}\, r_{is}, && \kc \in \Kc_i,\; i\in\Jc, \\
& r_{is} \geq R_{is}^{\kc} + (g_{is}^{\kc})^\top (\vartheta - \vartheta^{\kc}), && \kc \in \Kc_{is},\; s\in\Sc, i\in\Jc, \\
& \alpha, r_{is}, q_i \geq 0 \quad s\in\Sc, i\in\Jc.
\end{aligned}
\end{equation}

The three-stage method now is modified to the following three-level regularized multi-cut method.

   \begin{itemize}
    \item[]   \vspace{0.2cm} \hrule \vspace{0.2cm}
                \textbf{Three-level Classification with Contextual Systemic Risk} 
                \vspace{0.2cm} \hrule \vspace{0.2cm}  
\item[\emph{Step 0.}] Set $k=1$. Choose initial decision variable $\vartheta^1$.
\item[\emph{Step 1.}] For each  $(s,i)$, $s\in\Sc;\, i \in \Jc$, solve the third stage subproblem \eqref{p:fairness-IIIstage}. Denote its optimal value by $R_{is}^k$ and calculate the subgradient $g_{is}^k$. Add the new scenario cuts to $\Kc_{is}$.
\item[\emph{Step 2.}] For each class $i \in\Jc$, evaluate the class-level risk $W^k(i) = \varrho_i[R_{is}^k : s \in \Sc]$ and calculate $\zeta^{c,k}$ using \eqref{eq:class-dual}. Add the new class-level cuts to $\Kc_i$, $i\in\Jc$.
\item[\emph{Step 3.}] Calculate the systemic risk $\varrho_{\rm sys}^k = \varrho_0[W^k(\cdot)]$ and calculate $\mu^k$ using \eqref{eq:outer-dual}. Add the new outer cut to $\Kc_0$.
\item[\emph{Step 4.}] Determine the new center $w^k$ as follows. If $k=1$ or 
\[
\varrho_\sys^k \leq (1-\delta)\bar{\varrho}^{k-1} + \delta \alpha^{k-1},
\]
then set $w^k = \vartheta^k$ and $\bar{\varrho}^k = \varrho_\sys^k$ (descent step). Otherwise, set $w^k = w^{k-1}$ and $\bar{\varrho}^k = \bar{\varrho}^{k-1}$ (null step).
\item[\emph{Step 5.}] Solve the master problem \eqref{eq:master_3level}. Denote the solution by $\alpha^k, \vartheta^k, r^k, q^k$.
\item[\emph{Step 6.}] If $\bar{\varrho}^k = \alpha^k$, then stop ($w^k$ is an optimal solution); otherwise continue.
\item[\emph{Step 7.}] Remove from the sets $\Kc_0$, $\Kc_i$, and $\Kc_{is}$ the constraints whose Lagrange multipliers are 0 at the solution of \eqref{eq:master_3level}. Increase $k$ by 1 and go to Step~1.
  \item[]   \vspace{0.2cm} \hrule \vspace{0.2cm}
\end{itemize}

Observe that the convergence of both methods follows from the convergence theorem of the regularized decomposition method in \cite{dentcheva2025risk}.

\section{Numerical Experiments with Fair Risk-averse Classification}
\label{sub:fair_risk_averse_classification}

We now test our method on two datasets, the Adult dataset~\cite{adult_2} and the ACSPublicCoverage dataset from the Folktables benchmark suite~\cite{ding2021retiring}, for two different purposes. On first one, we study what happens when the recorded sensitive attribute is wrong by adding noises to the label. On the second one, since the ACSPublicCoverage data is recorded in different states, we can study what happens when the test set distribution is shifted from the training set. On both datasets we use a binary sensitive attribute (sex) and a multi-group one (race), and then, on the public coverage data, we use an intersectional attribute (sex $\times$ race), so we can thoroughly observe the behavior of the methods on different numbers of protected groups.

We compare CNACR against a few baseline methods: a plain linear SVM, a covariance-constrained classifier from \cite{zafar2019flexible}, a fair empirical risk minimization from \cite{donini2018empirical}, the reductions approach from \cite{agarwal2018reductionsapproachfairclassification} with EO constraints. All of them use a linear SVM as the base classifier, so if we see any difference, we know that it comes from the fairness-enforcing method but not from the classification model. Note that the method of \cite{donini2018empirical} is only defined for two protected groups, so we use it only in the experiments with sex as the sensitive attribute. We sweep through the parameter configurations for each baseline method and report the best one of each family in the tables. For CNACR we use fixed parameters $\kappa_{\mathrm{inner}} = 0.1$, $\kappa_{\mathrm{mid}} = 1.0$, $\kappa_{\mathrm{out}} = 0.5$.

We report the EO ratio as one of the fairness metrics. We argued previously that this mainstream metric is limited to reflecting the fairness between the two extreme groups when more than two groups are involved, so we also report the Gini coefficient and the goodness-of-fit test results.

\subsection{Adult: Corrupted Sensitive Attributes}

The Adult dataset has $48842$ samples, and the task is to predict whether an individual earns more than \$50K per year. We first use sex as the sensitive attribute, which is binary. Then we use race, grouped into White, Black, and Other with proportions $85\% / 10\% / 5\%$, which allows us to run the experiment in an imbalanced multi-group setting.

\paragraph{Setup.}
For each of $100$ independent runs, we randomly sample $20000$ data points from the full dataset proportionally across classes and genders, and split them into $70\%$ training and $30\%$ test. By 'proportionally', we mean that the sample keeps the same proportion of each class and each group as they are in the whole dataset. To simulate corrupted sensitive attributes, we randomly flip the sensitive attribute of $20\%$ of the training samples and keep the test set clean. For example, in the experiment with sex as the sensitive group, we randomly pick $20\%$ of the training samples, among which males become females and vice versa. All methods are trained on the same contaminated training set and evaluated on the same clean test set in each run, so each comparison is paired. We repeat this for 100 times at $0\%$ and $20\%$ mislabeling respectively, and report both results.

\paragraph{Results.}
Tables~\ref{tab:adult_sex} and~\ref{tab:adult_race} give the results of the experiment using sex as sensitive attribute. We report the average on all metrics $\pm$ standard deviation over the $100$ runs, and the best value in each row is bolded. Since each method uses the same train and test set in every run, the results are paired, and we run paired t-tests to see if the differences of the metrics are statistically significant.

\begin{table}[!htbp]
\centering
\footnotesize
\setlength{\tabcolsep}{2pt}
\begin{tabular}{llccccc}
\toprule
& & \textbf{SVM} & \textbf{CNACR} & \textbf{Zafar} & \textbf{Donini} & \textbf{FL-EOp}\\
\midrule
\multirow{4}{*}{\rotatebox{90}{Noise $0\%$}}
 & EO-ratio  & $0.8328 \pm 0.059$ & $\mathbf{0.9533 \pm 0.032}$ & $0.9480 \pm 0.038$ & $0.9270 \pm 0.055$ & $0.9391 \pm 0.043$ \\
 & Gini(TPR) & $0.0219 \pm 0.008$ & $\mathbf{0.0062 \pm 0.004}$ & $0.0069 \pm 0.005$ & $0.0095 \pm 0.007$ & $0.0080 \pm 0.006$ \\
 & F1        & $0.7813 \pm 0.006$ & $\mathbf{0.7866 \pm 0.006}$ & $0.7751 \pm 0.006$ & $0.7801 \pm 0.006$ & $0.7786 \pm 0.006$ \\
 & $\chi^2$ rej. & $76.0\%$ & $8.0\%$ & $\mathbf{6.0\%}$ & $20.0\%$ & $12.0\%$ \\
\midrule
\multirow{4}{*}{\rotatebox{90}{Noise $20\%$}}
 & EO-ratio  & $0.8328 \pm 0.059$ & $0.9463 \pm 0.040$ & $\mathbf{0.9466 \pm 0.039}$ & $0.9047 \pm 0.072$ & $0.9025 \pm 0.064$ \\
 & Gini(TPR) & $0.0219 \pm 0.008$ & $\mathbf{0.0071 \pm 0.005}$ & $\mathbf{0.0071 \pm 0.005}$ & $0.0124 \pm 0.009$ & $0.0127 \pm 0.008$ \\
 & F1        & $0.7813 \pm 0.006$ & $\mathbf{0.7871 \pm 0.006}$ & $0.7752 \pm 0.006$ & $0.7800 \pm 0.006$ & $0.7793 \pm 0.007$ \\
 & $\chi^2$ rej. & $76.0\%$ & $13.0\%$ & $\mathbf{8.0\%}$ & $33.0\%$ & $33.0\%$ \\
\bottomrule
\end{tabular}
\caption{Adult with sex as the sensitive attribute (mean $\pm$ standard deviation over $100$ runs). $\chi^2$ rej.\ is the fraction of runs in which equal TPR across groups is rejected at $\alpha = 0.05$. On this attribute, CNACR and the covariance-constrained classifier (Zafar) are statistically indistinguishable on both fairness measures; the bolding marks the better mean but with no significant difference.}
\label{tab:adult_sex}
\end{table}

\begin{table}[!htbp]
\centering
\footnotesize
\setlength{\tabcolsep}{3pt}
\begin{tabular}{llcccc}
\toprule
& & \textbf{SVM} & \textbf{CNACR} & \textbf{Zafar} & \textbf{FL-EOp} \\
\midrule
\multirow{4}{*}{\rotatebox{90}{Noise $0\%$}}
 & EO-ratio  & $0.7433 \pm 0.097$ & $\mathbf{0.8756 \pm 0.072}$ & $0.8581 \pm 0.075$ & $0.8322 \pm 0.096$ \\
 & Gini(TPR) & $0.0139 \pm 0.005$ & $\mathbf{0.0066 \pm 0.004}$ & $0.0082 \pm 0.005$ & $0.0092 \pm 0.005$ \\
 & F1        & $0.7813 \pm 0.005$ & $\mathbf{0.7892 \pm 0.006}$ & $0.7798 \pm 0.006$ & $0.7802 \pm 0.005$ \\
 & $\chi^2$ rej. & $54.0\%$ & $\mathbf{10.0\%}$ & $12.0\%$ & $16.0\%$ \\
\midrule
\multirow{4}{*}{\rotatebox{90}{Noise $20\%$}}
 & EO-ratio  & $0.7433 \pm 0.097$ & $\mathbf{0.8736 \pm 0.073}$ & $0.8548 \pm 0.079$ & $0.7938 \pm 0.111$ \\
 & Gini(TPR) & $0.0139 \pm 0.005$ & $\mathbf{0.0067 \pm 0.004}$ & $0.0084 \pm 0.005$ & $0.0113 \pm 0.006$ \\
 & F1        & $0.7813 \pm 0.005$ & $\mathbf{0.7893 \pm 0.005}$ & $0.7797 \pm 0.006$ & $0.7804 \pm 0.006$ \\
 & $\chi^2$ rej. & $54.0\%$ & $18.0\%$ & $\mathbf{14.0\%}$ & $35.0\%$ \\
\bottomrule
\end{tabular}
\caption{Adult with race as the sensitive attribute (mean $\pm$ standard deviation over $100$ runs). CNACR shows better EO-ratio, Gini and F1 than all baselines with statistical significance ($p < 0.01$).}
\label{tab:adult_race}
\end{table}
With race as the sensitive attribute (Table~\ref{tab:adult_race}), CNACR shows significant advantage. On clean data, compared to the plain SVM, it raises the EO-ratio from $0.7433$ to $0.8756$, more than halves the Gini coefficient from $0.0139$ to $0.0066$, and greatly reduces the chi-square rejection rate from a random $54\%$ to $10\%$. In a paired $t$-test, the advantage is significant with $p < 10^{-4}$, with CNACR higher in $92$ out of the $100$ runs.

The closest baseline is the covariance-constrained classifier (Zafar), and the two behave differently in the two experiments. In the race experiment, the paired differences favor CNACR on the EO-ratio ($+0.0175$, $p = 0.004$) and on the Gini coefficient ($-0.0017$, $p < 10^{-3}$). However, the two methods are statistically indistinguishable on both fairness measures in the sex experiment. But against every other baseline, and against the plain SVM, CNACR's fairness differences are still significant at the $0.05$ level on both attributes and at both corruption levels. This indicates that CNACR might show a greater advantage compared to the baseline methods as the number of attributes increases. The same pattern will repeat in future experiments.

This improvement in fairness metrics does not come at a cost in classification quality. CNACR attains a higher F1 score than every baseline on both attributes and at both corruption levels, and each of these differences is significant at $p < 10^{-4}$. 

\subsection*{ACSPublicCoverage: Distribution Shift across States}

In this section, we conduct the experiment on the ACSPublicCoverage dataset. We train all the method only on California data and test on all 8 states, for a distributional shift test. ACSPublicCoverage predicts whether an individual has public health insurance coverage. We choose this dataset because it is very different than the Adult dataset. The coverage rates in the raw data are very close across all groups. Same as the Adult dataset experiments, we first conduct the binary-group experiment using sex as the attribute. Then we conduct the multi-group experiment using the 5 race groups in the data: White, Black, Hispanic, Asian, and Other. In addition, we also use the intersection of the two (sex $\times$ race) for a 10-sensitive-attribute experiment.

\paragraph{Setup.}
For each run, we randomly sample $N_{\mathrm{train}} = 20000$ data proportionally from California. Again, the sampling preserves the proportion of each class and attribute. We only train the methods with California data, and then we randomly sample $N_{\mathrm{test}} = 5000$ out-of-sample data from each state, including California, as test sets. We repeat this for $100$ runs and report the mean $\pm$ standard deviation like before. And again, every method is trained on the same training set and tested on the same test set, allowing us to compare the results by paired $t$-tests over the $100$ runs.

\paragraph{Results.}
Tables~\ref{tab:pubcov_sex_eor}--\ref{tab:pubcov_sex_chi2} give the binary case with sex as the protected attribute. CNACR attains the best EO-ratio and the lowest Gini coefficient on every state. According to the paired $t$-tests, every advantage is significant at $p < 10^{-4}$, with CNACR ahead on over $90$ of the $100$ runs. On the contrary, the three fairness baselines stay very close to the plain SVM on both measures in every state. Table~\ref{tab:pubcov_sex_chi2} reports the goodness-of-fit test rejection counts at $\alpha = 0.05$. CNACR has the lowest or joint-lowest count on 6 of the 8 states. Pooled over all $800$ tests, CNACR rejects the null hypothesis of equal TPR in $4.9\%$ of cases, lowest compared to all the other methods. In terms of classification performance, CNACR has the lowest pooled F1 score. However, we can see that the poor performance is mainly in the last three states (FL, GA, TX), and CNACR has the best F1 score out of the five methods in the other five states, with a smaller standard deviation. The three losing states have very different distributions from the rest of the five states, whose coverage rates are much higher. This pattern shows again in the race experiment.

\begin{table}[!htbp]
\centering
\footnotesize
\setlength{\tabcolsep}{2pt}
\begin{tabular}{lccccc}
\toprule
\textbf{State} & \textbf{SVM} & \textbf{CNACR} & \textbf{Zafar} & \textbf{Donini} & \textbf{FL-EOp} \\
\midrule
CA & $0.9433 \pm 0.042$ & $\mathbf{0.9618 \pm 0.027}$ & $0.9434 \pm 0.039$ & $0.9453 \pm 0.039$ & $0.9444 \pm 0.041$ \\
NY & $0.9552 \pm 0.030$ & $\mathbf{0.9657 \pm 0.022}$ & $0.9545 \pm 0.031$ & $0.9565 \pm 0.030$ & $0.9551 \pm 0.031$ \\
MA & $0.9604 \pm 0.033$ & $\mathbf{0.9726 \pm 0.019}$ & $0.9587 \pm 0.034$ & $0.9589 \pm 0.034$ & $0.9607 \pm 0.034$ \\
AZ & $0.9594 \pm 0.030$ & $\mathbf{0.9705 \pm 0.020}$ & $0.9588 \pm 0.029$ & $0.9606 \pm 0.028$ & $0.9607 \pm 0.030$ \\
IL & $0.9521 \pm 0.035$ & $\mathbf{0.9669 \pm 0.028}$ & $0.9527 \pm 0.034$ & $0.9532 \pm 0.034$ & $0.9512 \pm 0.037$ \\
FL & $0.9327 \pm 0.039$ & $\mathbf{0.9628 \pm 0.027}$ & $0.9335 \pm 0.041$ & $0.9347 \pm 0.041$ & $0.9335 \pm 0.041$ \\
GA & $0.9569 \pm 0.029$ & $\mathbf{0.9723 \pm 0.020}$ & $0.9565 \pm 0.029$ & $0.9562 \pm 0.032$ & $0.9582 \pm 0.031$ \\
TX & $0.9440 \pm 0.038$ & $\mathbf{0.9649 \pm 0.024}$ & $0.9428 \pm 0.038$ & $0.9437 \pm 0.037$ & $0.9433 \pm 0.037$ \\
\midrule
\textbf{Pooled} & $0.9505$ & $\mathbf{0.9672}$ & $0.9501$ & $0.9511$ & $0.9509$ \\
\bottomrule
\end{tabular}
\caption{ACSPublicCoverage, sex: per-state EO-ratio (mean $\pm$ std over $100$ runs).}
\label{tab:pubcov_sex_eor}
\end{table}

\begin{table}[!htbp]
\centering
\footnotesize
\setlength{\tabcolsep}{2pt}
\begin{tabular}{lccccc}
\toprule
\textbf{State} & \textbf{SVM} & \textbf{CNACR} & \textbf{Zafar} & \textbf{Donini} & \textbf{FL-EOp} \\
\midrule
CA & $0.0148 \pm 0.011$ & $\mathbf{0.0098 \pm 0.007}$ & $0.0148 \pm 0.010$ & $0.0143 \pm 0.011$ & $0.0145 \pm 0.011$ \\
NY & $0.0116 \pm 0.008$ & $\mathbf{0.0088 \pm 0.006}$ & $0.0118 \pm 0.008$ & $0.0112 \pm 0.008$ & $0.0116 \pm 0.008$ \\
MA & $0.0102 \pm 0.009$ & $\mathbf{0.0070 \pm 0.005}$ & $0.0107 \pm 0.009$ & $0.0106 \pm 0.009$ & $0.0102 \pm 0.009$ \\
AZ & $0.0105 \pm 0.008$ & $\mathbf{0.0075 \pm 0.005}$ & $0.0106 \pm 0.008$ & $0.0101 \pm 0.007$ & $0.0101 \pm 0.008$ \\
IL & $0.0124 \pm 0.009$ & $\mathbf{0.0085 \pm 0.007}$ & $0.0123 \pm 0.009$ & $0.0121 \pm 0.009$ & $0.0127 \pm 0.010$ \\
FL & $0.0176 \pm 0.011$ & $\mathbf{0.0096 \pm 0.007}$ & $0.0174 \pm 0.011$ & $0.0171 \pm 0.011$ & $0.0174 \pm 0.011$ \\
GA & $0.0111 \pm 0.008$ & $\mathbf{0.0071 \pm 0.005}$ & $0.0112 \pm 0.008$ & $0.0113 \pm 0.008$ & $0.0108 \pm 0.008$ \\
TX & $0.0146 \pm 0.010$ & $\mathbf{0.0090 \pm 0.006}$ & $0.0149 \pm 0.010$ & $0.0147 \pm 0.010$ & $0.0148 \pm 0.010$ \\
\midrule
\textbf{Pooled} & $0.0128$ & $\mathbf{0.0084}$ & $0.0130$ & $0.0127$ & $0.0128$ \\
\bottomrule
\end{tabular}
\caption{ACSPublicCoverage, sex: per-state Gini coefficient on per-group TPRs (mean $\pm$ std over $100$ runs).}
\label{tab:pubcov_sex_gini}
\end{table}

\begin{table}[!htbp]
\centering
\footnotesize
\setlength{\tabcolsep}{2pt}
\begin{tabular}{lccccc}
\toprule
\textbf{State} & \textbf{SVM} & \textbf{CNACR} & \textbf{Zafar} & \textbf{Donini} & \textbf{FL-EOp} \\
\midrule
CA & $0.6270 \pm 0.008$ & $\mathbf{0.6514 \pm 0.007}$ & $0.6262 \pm 0.008$ & $0.6266 \pm 0.008$ & $0.6268 \pm 0.008$ \\
NY & $0.6448 \pm 0.008$ & $\mathbf{0.6691 \pm 0.007}$ & $0.6441 \pm 0.008$ & $0.6434 \pm 0.008$ & $0.6442 \pm 0.008$ \\
MA & $0.6885 \pm 0.008$ & $\mathbf{0.7100 \pm 0.007}$ & $0.6878 \pm 0.008$ & $0.6867 \pm 0.008$ & $0.6876 \pm 0.008$ \\
AZ & $0.6604 \pm 0.007$ & $\mathbf{0.6620 \pm 0.006}$ & $0.6599 \pm 0.007$ & $0.6598 \pm 0.007$ & $0.6601 \pm 0.007$ \\
IL & $0.6926 \pm 0.007$ & $\mathbf{0.6989 \pm 0.007}$ & $0.6912 \pm 0.007$ & $0.6911 \pm 0.008$ & $0.6921 \pm 0.007$ \\
FL & $0.6779 \pm 0.007$ & $0.6542 \pm 0.008$ & $0.6782 \pm 0.007$ & $\mathbf{0.6783 \pm 0.007}$ & $0.6779 \pm 0.007$ \\
GA & $0.6909 \pm 0.007$ & $0.6592 \pm 0.008$ & $0.6912 \pm 0.008$ & $\mathbf{0.6917 \pm 0.007}$ & $0.6912 \pm 0.007$ \\
TX & $0.6744 \pm 0.008$ & $0.6379 \pm 0.008$ & $0.6743 \pm 0.007$ & $\mathbf{0.6751 \pm 0.007}$ & $0.6746 \pm 0.007$ \\
\midrule
\textbf{Pooled} & $\mathbf{0.6696}$ & $0.6678$ & $0.6691$ & $0.6691$ & $0.6693$ \\
\bottomrule
\end{tabular}
\caption{ACSPublicCoverage, sex: per-state F1 score (mean $\pm$ std over $100$ runs).}
\label{tab:pubcov_sex_f1}
\end{table}

\begin{table}[!htbp]
\centering
\footnotesize
\setlength{\tabcolsep}{4pt}
\begin{tabular}{lccccc}
\toprule
\textbf{State} & \textbf{SVM} & \textbf{CNACR} & \textbf{Zafar} & \textbf{Donini} & \textbf{FL-EOp} \\
\midrule
CA & $10/100$ & $\mathbf{6/100}$ & $8/100$ & $7/100$ & $9/100$ \\
NY & $5/100$ & $\mathbf{4/100}$ & $6/100$ & $\mathbf{4/100}$ & $7/100$ \\
MA & $8/100$ & $\mathbf{2/100}$ & $7/100$ & $7/100$ & $9/100$ \\
AZ & $\mathbf{2/100}$ & $\mathbf{2/100}$ & $\mathbf{2/100}$ & $\mathbf{2/100}$ & $\mathbf{2/100}$ \\
IL & $8/100$ & $10/100$ & $6/100$ & $\mathbf{5/100}$ & $9/100$ \\
FL & $10/100$ & $\mathbf{7/100}$ & $11/100$ & $8/100$ & $11/100$ \\
GA & $3/100$ & $3/100$ & $\mathbf{1/100}$ & $4/100$ & $5/100$ \\
TX & $6/100$ & $\mathbf{5/100}$ & $7/100$ & $6/100$ & $\mathbf{5/100}$ \\
\midrule
\textbf{Total} & $6.5\%$ & $\mathbf{4.9\%}$ & $6.0\%$ & $5.4\%$ & $7.1\%$ \\
\bottomrule
\end{tabular}
\caption{ACSPublicCoverage, sex: per-state chi-square rejection counts out of $100$ runs, at $\alpha=0.05$. The last row pools all $800$ (seed, state) settings and is descriptive only, since the eight states within a run share one trained classifier.}
\label{tab:pubcov_sex_chi2}
\end{table}

We now report the five-group case with race as the sensitive attribute. CNACR attains the best EO-ratio on every test state (Table~\ref{tab:pubcov_eor}) with the lowest standard deviations, so CNACR's fairness is not only better on average but also more consistent from run to run. And the paired $t$-test says the difference over the $100$ runs is significant at $p < 10^{-4}$ against every baseline, with CNACR winning on $96$ to $98$ of the $100$ runs. 

\begin{table}[!htbp]
\centering
\footnotesize
\setlength{\tabcolsep}{3pt}
\begin{tabular}{lcccc}
\toprule
\textbf{State} & \textbf{SVM} & \textbf{CNACR} & \textbf{Zafar} & \textbf{FL-EOp} \\
\midrule
CA & $0.7368 \pm 0.089$ & $\mathbf{0.8311 \pm 0.057}$ & $0.7572 \pm 0.087$ & $0.7759 \pm 0.089$ \\
TX & $0.7475 \pm 0.100$ & $\mathbf{0.8216 \pm 0.082}$ & $0.7472 \pm 0.099$ & $0.7201 \pm 0.103$ \\
NY & $0.7891 \pm 0.090$ & $\mathbf{0.8352 \pm 0.078}$ & $0.7923 \pm 0.091$ & $0.7663 \pm 0.103$ \\
FL & $0.7012 \pm 0.132$ & $\mathbf{0.7900 \pm 0.098}$ & $0.7081 \pm 0.129$ & $0.7005 \pm 0.128$ \\
IL & $0.7517 \pm 0.089$ & $\mathbf{0.8098 \pm 0.073}$ & $0.7584 \pm 0.092$ & $0.7461 \pm 0.097$ \\
MA & $0.7583 \pm 0.092$ & $\mathbf{0.8112 \pm 0.068}$ & $0.7432 \pm 0.093$ & $0.6881 \pm 0.111$ \\
GA & $0.7654 \pm 0.099$ & $\mathbf{0.8445 \pm 0.071}$ & $0.7712 \pm 0.100$ & $0.7462 \pm 0.096$ \\
AZ & $0.7291 \pm 0.095$ & $\mathbf{0.8083 \pm 0.081}$ & $0.7480 \pm 0.090$ & $0.7514 \pm 0.091$ \\
\midrule
\textbf{Pooled} & $0.7474$ & $\mathbf{0.8190}$ & $0.7532$ & $0.7368$ \\
\bottomrule
\end{tabular}
\caption{ACSPublicCoverage, race: per-state EO-ratio (mean $\pm$ std over $100$ runs).}
\label{tab:pubcov_eor}
\end{table}

\begin{table}[!htbp]
\centering
\footnotesize
\setlength{\tabcolsep}{3pt}
\begin{tabular}{lcccc}
\toprule
\textbf{State} & \textbf{SVM} & \textbf{CNACR} & \textbf{Zafar} & \textbf{FL-EOp} \\
\midrule
CA & $0.6270 \pm 0.007$          & $\mathbf{0.6514 \pm 0.006}$ & $0.6268 \pm 0.007$ & $0.6265 \pm 0.007$ \\
TX & $\mathbf{0.6745 \pm 0.007}$ & $0.6388 \pm 0.008$ & $0.6745 \pm 0.008$ & $0.6742 \pm 0.008$ \\
NY & $0.6450 \pm 0.007$          & $\mathbf{0.6692 \pm 0.007}$ & $0.6445 \pm 0.007$ & $0.6434 \pm 0.007$ \\
FL & $\mathbf{0.6783 \pm 0.007}$ & $0.6544 \pm 0.007$ & $0.6786 \pm 0.007$ & $0.6782 \pm 0.007$ \\
IL & $0.6926 \pm 0.007$          & $\mathbf{0.7002 \pm 0.007}$ & $0.6920 \pm 0.007$ & $0.6904 \pm 0.007$ \\
MA & $0.6879 \pm 0.008$          & $\mathbf{0.7109 \pm 0.008}$ & $0.6875 \pm 0.008$ & $0.6875 \pm 0.008$ \\
GA & $\mathbf{0.6906 \pm 0.008}$ & $0.6600 \pm 0.008$ & $0.6909 \pm 0.008$ & $0.6900 \pm 0.008$ \\
AZ & $0.6608 \pm 0.007$          & $\mathbf{0.6628 \pm 0.007}$ & $0.6607 \pm 0.007$ & $0.6603 \pm 0.007$ \\
\midrule
\textbf{Pooled} & $\mathbf{0.6696}$ & $0.6685$ & $0.6694$ & $0.6688$ \\
\bottomrule
\end{tabular}
\caption{ACSPublicCoverage, race: per-state F1 score (mean $\pm$ std over $100$ runs).}
\label{tab:pubcov_f1}
\end{table}

\begin{table}[!htbp]
\centering
\footnotesize
\setlength{\tabcolsep}{3pt}
\begin{tabular}{lcccc}
\toprule
\textbf{State} & \textbf{SVM} & \textbf{CNACR} & \textbf{Zafar} & \textbf{FL-EOp} \\
\midrule
CA & $0.0351 \pm 0.012$ & $\mathbf{0.0216 \pm 0.008}$ & $0.0325 \pm 0.012$ & $0.0287 \pm 0.011$ \\
TX & $0.0259 \pm 0.009$ & $\mathbf{0.0182 \pm 0.007}$ & $0.0260 \pm 0.010$ & $0.0295 \pm 0.011$ \\
NY & $0.0227 \pm 0.009$ & $\mathbf{0.0169 \pm 0.007}$ & $0.0220 \pm 0.009$ & $0.0255 \pm 0.010$ \\
FL & $0.0253 \pm 0.011$ & $\mathbf{0.0176 \pm 0.008}$ & $0.0249 \pm 0.010$ & $0.0271 \pm 0.011$ \\
IL & $0.0274 \pm 0.010$ & $\mathbf{0.0186 \pm 0.007}$ & $0.0252 \pm 0.010$ & $0.0241 \pm 0.008$ \\
MA & $0.0204 \pm 0.008$ & $\mathbf{0.0151 \pm 0.005}$ & $0.0215 \pm 0.008$ & $0.0262 \pm 0.009$ \\
GA & $0.0227 \pm 0.009$ & $\mathbf{0.0144 \pm 0.007}$ & $0.0225 \pm 0.009$ & $0.0331 \pm 0.014$ \\
AZ & $0.0267 \pm 0.011$ & $\mathbf{0.0171 \pm 0.007}$ & $0.0255 \pm 0.011$ & $0.0263 \pm 0.011$ \\
\midrule
\textbf{Pooled} & $0.0258$ & $\mathbf{0.0174}$ & $0.0250$ & $0.0276$ \\
\bottomrule
\end{tabular}
\caption{ACSPublicCoverage, race: per-state Gini coefficient on per-group TPRs (mean $\pm$ std over $100$ runs).}
\label{tab:pubcov_gini}
\end{table}

\begin{table}[!htbp]
\centering
\footnotesize
\setlength{\tabcolsep}{4pt}
\begin{tabular}{lcccc}
\toprule
\textbf{State} & \textbf{SVM} & \textbf{CNACR} & \textbf{Zafar} & \textbf{FL-EOp} \\
\midrule
CA & $21/100$ & $8/100$          & $16/100$ & $\mathbf{6/100}$  \\
TX & $5/100$  & $\mathbf{2/100}$ & $5/100$  & $13/100$ \\
NY & $\mathbf{3/100}$ & $4/100$  & $6/100$  & $13/100$ \\
FL & $10/100$ & $\mathbf{6/100}$ & $9/100$  & $12/100$ \\
IL & $22/100$ & $17/100$         & $13/100$ & $\mathbf{10/100}$ \\
MA & $6/100$  & $\mathbf{5/100}$ & $11/100$ & $35/100$ \\
GA & $7/100$  & $\mathbf{5/100}$ & $6/100$  & $22/100$ \\
AZ & $14/100$ & $12/100$         & $\mathbf{10/100}$ & $\mathbf{10/100}$ \\
\midrule
\textbf{Total} & $88/800\;(11.0\%)$ & $\mathbf{59/800\;(7.4\%)}$ & $76/800\;(9.5\%)$ & $119/800\;(14.9\%)$ \\
\bottomrule
\end{tabular}
\caption{ACSPublicCoverage, race: per-state chi-square rejection counts, that is, the number of runs (out of $100$) in which the test rejects $H_0$ (equal TPRs across the 5 race groups) at $\alpha = 0.05$. The last row pools across all $100 \times 8 = 800$ (seed, state) settings; since the eight states within a run share one trained classifier, this pooled figure is descriptive and is not used for inference.}
\label{tab:pubcov_chi2}
\end{table}

The F1 difference is small and depends on the state (Table~\ref{tab:pubcov_f1}). CNACR is the best on CA, NY, IL, MA and AZ, but it is the worst on TX, FL and GA, just like the experiment with the binary sex attribute.

The Gini coefficient results (Table~\ref{tab:pubcov_gini}) again show that CNACR has the lowest score on all 8 states, and the paired $t$-test shows the difference is significant against every baseline at $p < 10^{-4}$, with CNACR ahead on $98$ or $99$ of the $100$ runs. Table~\ref{tab:pubcov_chi2} reports the goodness-of-fit test rejection counts at $\alpha = 0.05$. CNACR has the lowest or joint-lowest count on 4 of the 8 states. Pooled over all $800$ (seed, state) settings, CNACR rejects the null hypothesis in $7.4\%$ of cases, which is again the lowest among all the methods.

We can see that the reduction method form the fairness constraint only on the CA training distribution, and once the deployment population changes they make fairness worse than a plain SVM. This shows an overfitting pattern of the fairness constraint to the source distribution. On the other hand, the other two baseline classifiers sit in between, neither helping nor hurting. The coverage rates across the groups are too close; therefore, the constraints of these two methods leave their classifiers almost unchanged. CNACR is the only method that improves the fairness in this situation.

\subsection*{Intersectional Groups: Sex $\times$ Race}

The two attributes above are treated separately. We now take their intersection, which gives $10$ protected groups. The coverage rates of the ten intersectional groups in California spread from $0.360$ to $0.394$, slightly wider than the spread ot the five race groups difference. Therefore, this does not introduce any new disparity in the raw data; it is still in the same situation as the previous two experiments. The experiment protocol is identical to the previous experiments: $100$ runs, training on only California, $N_{\mathrm{train}} = 20000$ and $N_{\mathrm{test}} = 5000$ per state.

Tables~\ref{tab:pubcov_sr_eor}--\ref{tab:pubcov_sr_chi2} report the results.

\begin{table}[!htbp]
\centering
\footnotesize
\setlength{\tabcolsep}{3pt}
\begin{tabular}{lccccc}
\toprule
\textbf{State} & \textbf{SVM} & \textbf{CNACR} & \textbf{Zafar} & \textbf{Donini} & \textbf{FL-EOp} \\
\midrule
CA & $0.5947$ & $\mathbf{0.7057}$ & $0.6070$ & $0.6150$ & $0.5873$ \\
NY & $0.6251$ & $\mathbf{0.6956}$ & $0.6064$ & $0.6019$ & $0.5748$ \\
MA & $0.5875$ & $\mathbf{0.6655}$ & $0.5722$ & $0.5512$ & $0.5246$ \\
AZ & $0.5659$ & $\mathbf{0.6436}$ & $0.5780$ & $0.5687$ & $0.5566$ \\
IL & $0.5774$ & $\mathbf{0.6618}$ & $0.5722$ & $0.5665$ & $0.5330$ \\
FL & $0.4853$ & $\mathbf{0.6087}$ & $0.4826$ & $0.4917$ & $0.4617$ \\
GA & $0.5573$ & $\mathbf{0.6835}$ & $0.5592$ & $0.5635$ & $0.5220$ \\
TX & $0.5252$ & $\mathbf{0.6583}$ & $0.5242$ & $0.5233$ & $0.4770$ \\
\midrule
\textbf{Pooled} & $0.5648$ & $\mathbf{0.6653}$ & $0.5627$ & $0.5602$ & $0.5296$ \\
\bottomrule
\end{tabular}
\caption{ACSPublicCoverage, sex $\times$ race ($10$ groups): per-state EO-ratio (mean over $100$ runs).}
\label{tab:pubcov_sr_eor}
\end{table}

\begin{table}[!htbp]
\centering
\footnotesize
\setlength{\tabcolsep}{3pt}
\begin{tabular}{lccccc}
\toprule
\textbf{State} & \textbf{SVM} & \textbf{CNACR} & \textbf{Zafar} & \textbf{Donini} & \textbf{FL-EOp} \\
\midrule
CA & $0.0519$ & $\mathbf{0.0347}$ & $0.0499$ & $0.0492$ & $0.0502$ \\
NY & $0.0378$ & $\mathbf{0.0286}$ & $0.0388$ & $0.0395$ & $0.0443$ \\
MA & $0.0357$ & $\mathbf{0.0270}$ & $0.0370$ & $0.0374$ & $0.0414$ \\
AZ & $0.0411$ & $\mathbf{0.0281}$ & $0.0407$ & $0.0416$ & $0.0442$ \\
IL & $0.0431$ & $\mathbf{0.0311}$ & $0.0432$ & $0.0419$ & $0.0461$ \\
FL & $0.0441$ & $\mathbf{0.0298}$ & $0.0447$ & $0.0439$ & $0.0485$ \\
GA & $0.0373$ & $\mathbf{0.0265}$ & $0.0385$ & $0.0402$ & $0.0576$ \\
TX & $0.0450$ & $\mathbf{0.0323}$ & $0.0461$ & $0.0468$ & $0.0526$ \\
\midrule
\textbf{Pooled} & $0.0420$ & $\mathbf{0.0298}$ & $0.0424$ & $0.0426$ & $0.0481$ \\
\bottomrule
\end{tabular}
\caption{ACSPublicCoverage, sex $\times$ race: per-state Gini coefficient on per-group TPRs (mean over $100$ runs).}
\label{tab:pubcov_sr_gini}
\end{table}

\begin{table}[!htbp]
\centering
\footnotesize
\setlength{\tabcolsep}{3pt}
\begin{tabular}{lccccc}
\toprule
\textbf{State} & \textbf{SVM} & \textbf{CNACR} & \textbf{Zafar} & \textbf{Donini} & \textbf{FL-EOp} \\
\midrule
CA & $0.6271$ & $\mathbf{0.6512}$ & $0.6242$ & $0.6241$ & $0.6172$ \\
NY & $0.6450$ & $\mathbf{0.6681}$ & $0.6408$ & $0.6403$ & $0.6299$ \\
MA & $0.6884$ & $\mathbf{0.7100}$ & $0.6840$ & $0.6828$ & $0.6729$ \\
AZ & $0.6610$ & $\mathbf{0.6627}$ & $0.6583$ & $0.6579$ & $0.6480$ \\
IL & $0.6927$ & $\mathbf{0.6999}$ & $0.6887$ & $0.6872$ & $0.6762$ \\
FL & $\mathbf{0.6785}$ & $0.6556$ & $0.6773$ & $0.6770$ & $0.6662$ \\
GA & $0.6914$ & $0.6620$ & $\mathbf{0.6916}$ & $0.6903$ & $0.6743$ \\
TX & $0.6745$ & $0.6398$ & $\mathbf{0.6746}$ & $0.6740$ & $0.6636$ \\
\midrule
\textbf{Pooled} & $\mathbf{0.6698}$ & $0.6687$ & $0.6674$ & $0.6667$ & $0.6560$ \\
\bottomrule
\end{tabular}
\caption{ACSPublicCoverage, sex $\times$ race: per-state F1 score (mean over $100$ runs).}
\label{tab:pubcov_sr_f1}
\end{table}

\begin{table}[!htbp]
\centering
\footnotesize
\setlength{\tabcolsep}{4pt}
\begin{tabular}{lccccc}
\toprule
\textbf{State} & \textbf{SVM} & \textbf{CNACR} & \textbf{Zafar} & \textbf{Donini} & \textbf{FL-EOp} \\
\midrule
CA & $14/100$ & $10/100$ & $11/100$ & $11/100$ & $\mathbf{5/100}$ \\
NY & $\mathbf{4/100}$ & $\mathbf{4/100}$ & $5/100$ & $5/100$ & $11/100$ \\
MA & $\mathbf{4/100}$ & $5/100$ & $5/100$ & $10/100$ & $17/100$ \\
AZ & $13/100$ & $\mathbf{8/100}$ & $\mathbf{8/100}$ & $9/100$ & $10/100$ \\
IL & $10/100$ & $10/100$ & $13/100$ & $\mathbf{4/100}$ & $13/100$ \\
FL & $8/100$ & $\mathbf{5/100}$ & $8/100$ & $11/100$ & $20/100$ \\
GA & $6/100$ & $\mathbf{4/100}$ & $7/100$ & $\mathbf{4/100}$ & $33/100$ \\
TX & $5/100$ & $5/100$ & $\mathbf{4/100}$ & $5/100$ & $16/100$ \\
\midrule
\textbf{Total} & $8.0\%$ & $\mathbf{6.4\%}$ & $7.6\%$ & $7.4\%$ & $15.6\%$ \\
\bottomrule
\end{tabular}
\caption{ACSPublicCoverage, sex $\times$ race: per-state chi-square rejection counts out of $100$ runs, at $\alpha=0.05$.}
\label{tab:pubcov_sr_chi2}
\end{table}

In terms of classification, the F1 score of CNACR is the best on 5 out of 8 states, and it is lossing on the same three states again. The pooled F1 score is very close to the plain SVM, so the fairness improvement is obtained at essentially no cost in classification quality.

On the fairness metrics, CNACR attains the best EO-ratio and the lowest Gini coefficient in each of the eight states. Moreover, we notice that the advantage here is larger than in the five-group race experiment. This is the same pattern showed in \cite{dentcheva2025risk}, where the risk-averse method's advantage grows with the number of classes. In this case, CNACR's advantage grows with the number of protected groups.

The baselines move in the opposite direction. On the Gini coefficient, both two fairness baselines are worse than the plain SVM. The same ordering appears in the goodness-of-fit rejection rate table. It seems that constraining ten group-conditional rates simultaneously, several of which are estimated from very few positive examples, may overfit the constraint.

\subsection*{Remarks on $\kappa$ Parameters}

CNACR is reported at $\kappa_{\mathrm{inner}} = 0.1$, $\kappa_{\mathrm{mid}} = 1$ and $\kappa_{\mathrm{out}} = 0.5$ throughout the experiments. Note that $\kappa$ has a range of $[0,1]$ for the mean-semi-deviation to be a coherent risk measure. Looking at the two outer layers of the CNACR structure, $\kappa_{\mathrm{out}}$ aggregates the risk of each class, whereas $\kappa_{\mathrm{mid}}$ aggregates the risk of each sensitive group. This means that, theoretically, $\kappa_{\mathrm{mid}}$ should be the parameter that controls the strength of fairness enforcement. To see how the result changes with the choice of $\kappa$ parameters, we vary one parameter at a time over $[0,1]$ with the other two held fixed at the default value, using $20$ runs. 

The experiment shows that the layers contribute very differently, and it depends on the data. When the group base rates differ, the middle level produces the improvement: on Adult dataset, where the raw data carries a relatively great unfairness, raising $\kappa_{\mathrm{mid}}$ from $0$ to $1$ increases the EO-ratio from $0.8508$ to $0.9498$ in the sex experiment, $0.7640$ to $0.8394$ in the race experiment. And the Gini coefficient also decreases from $0.0194$ to $0.0066$ in the sex experiment, $0.0132$ to $0.0083$ in the race experiment. However, on ACSPubcov dataset, where the base rate are very close across the groups, $\kappa_{\mathrm{mid}}$  barely does anything. Under all settings, increasing the value of $\kappa_{\mathrm{mid}}$ leads to almost zero change in the fairness metrics. 

The outer layer is also contributing to fairness enforcement. Changing $\kappa_{\mathrm{out}}$ from $0$ to $1$ decreases Gini coefficient from $0.0087$ to $0.0060$ in the Adult dataset using sex as the attribute, which is a much smaller change compared to changing $\kappa_{\mathrm{mid}}$ under the same setting. However, with more groups involved and more scenarios to be aggregated, $\kappa_{\mathrm{out}}$ is contributing much more. In the race experiment with the Adult dataset, increasing $\kappa_{\mathrm{out}}$ decreases Gini coefficient from $0.0113$ to $0.0054$, which is even more significant than changing $\kappa_{\mathrm{mid}}$. And in the ACSPubcov dataset, where $\kappa_{\mathrm{mid}}$ does almost nothing, changing $\kappa_{\mathrm{out}}$ from $0$ to $1$ decreases Gini coefficient from $0.0145$ to $0.0076$ in the sex experiment, $0.0311$ to $0.0140$ in the race experiment.

On the performance part, increasing $\kappa_{\mathrm{mid}}$ for stronger fairness enforcement almost never costs anything. Under all settings, increasing $\kappa_{\mathrm{mid}}$ decreases the F1 score by a small amount that is indistinguishable from $0$. Increasing $\kappa_{\mathrm{out}}$, on the other hand, slightly helps with the F1 score in the ACSPubcov experiment. But it does not make any significant changes in the Adult experiment. 

\subsection*{Within-Group Unfairness}

\begin{table}[!htbp]
\centering
\footnotesize
\setlength{\tabcolsep}{4pt}
\begin{tabular}{llcccccc}
\toprule
\textbf{Attribute} & \textbf{Noise} & \textbf{SVM} & \textbf{CNACR} & \textbf{Zafar} & \textbf{Donini} & \textbf{FL-EOp} & \textbf{FL-EOdds}\\
\midrule
\multirow{2}{*}{sex}  & $0\%$  & $0.08236$ & $\mathbf{0.08019}$ & $0.08382$ & $0.08257$ & $0.08315$ & $0.08883$ \\
                      & $20\%$ & $0.08236$ & $\mathbf{0.07970}$ & $0.08384$ & $0.08255$ & $0.08332$ & $0.08593$ \\
\midrule
\multirow{2}{*}{race} & $0\%$  & $0.08269$ & $\mathbf{0.07820}$ & $0.08317$ & --- & $0.08303$ & $0.08383$ \\
                      & $20\%$ & $0.08269$ & $\mathbf{0.07808}$ & $0.08315$ & --- & $0.08288$ & $0.08370$ \\
\bottomrule
\end{tabular}
\caption{Adult: within-group unfairness $I_{\mathrm{within}}$ under the Generalized Entropy decomposition of~\cite{speicher2018unified} (mean over $100$ runs). Lower is better. The method of~\cite{donini2018empirical} is defined for two groups only and is not used in the race experiment.}
\label{tab:ge_adult}
\end{table}

\begin{table}[!htbp]
\centering
\footnotesize
\setlength{\tabcolsep}{4pt}
\begin{tabular}{lccccc}
\toprule
\textbf{Attribute} & \textbf{SVM} & \textbf{CNACR} & \textbf{Zafar} & \textbf{Donini} & \textbf{FL-EOp}\\
\midrule
sex  & $0.14526$ & $\mathbf{0.12941}$ & $0.14563$ & $0.14549$ & $0.14531$ \\
race & $0.14444$ & $\mathbf{0.12942}$ & $0.14456$ & --- & $0.14498$ \\
\bottomrule
\end{tabular}
\caption{ACSPublicCoverage: within-group unfairness $I_{\mathrm{within}}$, pooled over the 8 test states (mean over $100$ runs). Lower is better.}
\label{tab:ge_pubcov}
\end{table}

All the measures used so far compare the true positive rates of the groups against each other, so they only see the group averages. \cite{speicher2018unified} point out that this can hide a second kind of unfairness: a method can equalize the group averages and at the same time make the treatment of the individuals inside a group more uneven. They give each individual a benefit score $b_i = \hat{y}_i - y_i + 1$, so a false negative receives $0$, a correct prediction receives $1$ and a false positive receives $2$, and measure the inequality of the benefit vector with the Generalized Entropy index at $\alpha = 2$. This index splits exactly into a between-group part and a within-group part,
\[
   I_{\mathrm{total}} \;=\; I_{\mathrm{between}} \;+\; I_{\mathrm{within}},
\]
where $I_{\mathrm{between}}$ replaces every individual by the mean benefit of their group and $I_{\mathrm{within}}$ combines the internal inequality of each group. We verified the identity numerically on every run and the largest residual was below $10^{-15}$. The between-group part measures the same thing as the fairness measures already reported in the previous tables, namely how the groups compare on average, so here we report only the within-group part, which the earlier measures cannot see.

Tables~\ref{tab:ge_adult} and~\ref{tab:ge_pubcov} give the results. CNACR is the only in-training method that reduces the within-group unfairness relative to the plain SVM, and it does so in all four settings: by $2.6\%$ and $3.2\%$ on Adult with sex, by $5.4\%$ and $5.6\%$ on Adult with race, and by about $11\%$ and $10\%$ on ACSPublicCoverage. Every constraint-based method increases it, in every setting, although the increases are small. This means that, for CNACR, the fairness improvements reported in the previous sections are not obtained at the cost of treating individuals inside a group more unevenly, which \cite{speicher2018unified} warn about for methods that equalize only the group averages.

\FloatBarrier

\bibliographystyle{plain}
\bibliography{bibliography}

\end{document}